\documentclass{article}

 \usepackage[final]{nips_2016}

\usepackage[utf8]{inputenc} 
\usepackage[T1]{fontenc}    
\usepackage{hyperref}       
\usepackage{url}            
\usepackage{booktabs}       
\usepackage{amsfonts}       
\usepackage{bm}
\usepackage{nicefrac}       
\usepackage{microtype}      
\usepackage{blindtext, graphicx}
\usepackage{float}
\usepackage{hyperref}
\usepackage{subcaption}
\usepackage{amsmath,amssymb,amsfonts,mathrsfs,amsthm}
\usepackage{pseudocode}
\usepackage{listings}

\usepackage{color}

\definecolor{gray}{rgb}{0.5,0.5,0.5}
\definecolor{mauve}{rgb}{0.58,0,0.82}

\def \d {\textbf{d}}

\DeclareMathOperator*{\argmin}{arg\,min}
\DeclareMathOperator*{\argmax}{arg\,max}

\newtheorem{theorem}{Theorem}

\usepackage[]{algorithm2e}
\usepackage{multicol}
\usepackage{xcolor}

\title{Multiscale Residual Mixture of PCA:\\Dynamic Dictionaries for Optimal Basis Learning}

\author{
  Randall BALESTRIERO \\
  ECE Department\\
  Rice University\\
  HOUSTON, TX\\
  \texttt{randallbalestriero@gmail.com} \\
}

\begin{document}

%

\maketitle

\begin{abstract}
  In this paper we are interested in the problem of learning an over-complete basis and a methodology such that the reconstruction or inverse problem does not need optimization. We analyze the optimality of the presented approaches, their link to popular already known techniques s.a. Artificial Neural Networks,k-means or Oja's learning rule. Finally, we will see that one approach to reach the optimal dictionary is a factorial and hierarchical approach. The derived approach lead to a formulation of a Deep Oja Network. We present results on different tasks and present the resulting very efficient learning algorithm which brings a new vision on the training of deep nets. Finally, the theoretical work shows that deep frameworks are one way to efficiently have over-complete (combinatorially large) dictionary yet allowing easy reconstruction. We thus present the Deep Residual Oja Network (DRON). We demonstrate that a recursive deep approach working on the residuals allow exponential decrease of the error w.r.t. the depth.
\end{abstract}

\section{Introduction}
We first discuss briefly the problem of dictionary learning and the different choices one has to make to perform this task depending on the problem at hand. We also introduce notations and standard operators.
\subsection{Why a Dictionary ?}

In order to analyze a given finite dataset $X:=\{x_n \in \mathbb{R}^D,n=1,..,N\}$ different approaches are possible. One of them lies on the assumption that those observations actually come from some latent representation that are mixed together, different mixing leading to different observations but with fixed dictionary $\Phi_K:=\{\phi_k \in \mathbb{R}^D,k=1,...,K\}$ with usually $K \ll N$. One can thus rewrite 
\begin{equation}\label{eq0}
    x_n = f(\Phi,x_n).
\end{equation}
In practice, a linear assumption is used for $f$ and we write the prediction as 
\begin{equation}
    \hat{x}_n=\hat{f}(\hat{\Phi}_K,x_n),
\end{equation}
where one estimate the functional and the filter-bank through a reconstruction loss defined as
\begin{equation}
    E_n = ||x_n-\hat{x}_n ||^2,
\end{equation}
where we assume here a squared error but any loss function can be used in general.
The linear assumption imposed on $f$ leads to the estimation weighting coefficients we denote by $\alpha_{n,k}$ weighting for observation $n$ the atom $k$, and those attributes are the new features used to represent the observations.

This new representation of $x_n$ via its corresponding feature vector $\bm{\alpha_n}$ can be used for many tasks s.a. clustering, denoising, data generation, anomaly detection, compression and much more.
Depending on the constraints one imposes on the feature vectors $\bm{\alpha_n}$ and the dictionary $\Phi_K$, one can come back to standard known frameworks s.a. Principal Component Analysis (PCA)\cite{jolliffe2002principal}, Independent Component Analysis (ICA)\cite{hyvarinen2004independent}, Sparse Coding (SC)\cite{olshausen1997sparse}, (Semi) Nonegative Matrix Factorization (sNMF)\cite{lee1999learning,ding2010convex},  Gaussian Mixture Model (GMM)\cite{bilmes1998gentle}, and many more, but all those approaches can be categorized into two main categories: Complete Dictionary Learning and Over-Complete Dictionary Learning.

The use of a dictionary also extends to standard template matching, the mot popular technique and optimal in the GLRT sense, where new examples are to be mapped to one of the template to be clustered or else.

\subsection{(Over)complete Dictionary Learning}
As we saw, the dictionary learning problem finds many formulations but the main difference resides in the properties of the learned dictionary, namely if it is complete or over-complete. 
The general case of complete or orthogonal basis imposes the following constraint on the filter-bank:
\begin{align}
    <\phi_j,\phi_k>=0,\forall j \not = k,\;K=D,
\end{align}
and with sometimes the complementary constraints that $||\phi_k||=1,\forall k$ leads to an orthonormal basis. The orthogonality of the atoms allow to have exact reconstruction leading to $E_n=0,\forall n$ as by definition one has 
\begin{equation}\label{eq1}
    \hat{x}_n=\sum_i \frac{<x_n,\hat{\phi}_k>}{||\hat{\phi}_k||^2}\hat{\phi}_k,\forall n.
\end{equation}
Given a dictionary $\Phi$ this decomposition is unique and thus we have $(\hat{\Phi}_K,x_n) \Rightarrow \hat{\bm{\alpha}}_n,\forall n$.
However, while through the Gram-Schmidt process existence of such a basis is guaranteed, it is not unique.

On the other hand, $\Phi_K$ can be an over-complete basis with the main difference coming that now $K>D$ and with the extreme case of $K=N$. Containing more atoms than the dimension of the space leads to interesting properties in practice such as sparse and independent features (ICA), clustering properties (K-means, NMF), biological understanding as is has been empirically shown that visual and auditory cortex of many mammals contains over-complete dictionaries.
Yet, all those benefits coming from the redundancy of the atoms also lead to the non-unique features $\bm{\alpha}_n$ even when the dictionary is kept fixed, thus we have that $(\hat{\Phi}_K,x_n) \not \Rightarrow \hat{\bm{\alpha}}_n$.
As a result, one has to solve the additional optimization problem of
\begin{equation}
    \bm{\hat{\alpha}_n}=\argmin_{\bm{\alpha}\in \Omega \subset \mathbb{R}^K} ||x_n-\sum_k\alpha_k\hat{\phi}_k ||.
\end{equation}

As a result, whether one is in a complete or over-complete dictionary setting, the optimization problems are always ill-posed by the non-uniqueness of the solutions forcing to impose additional structures or constraints in order to reach well posed problems. For the complete case, the main approach consist in imposing that as few atoms as possible are used leading to PCA, a very powerful approach for dimensionality reduction and compression. For over-complete cases, different sparsity criteria are imposed on the features $\bm{\alpha}_n$ such as norm-(0,1,2). For a norm-0 we are in the Matching pursuit case, norm1 is sparse coding and norm2 is ridge regression.
For each of those cases many very efficient exact or iterative optimization algorithms have been developed to estimate $\hat{\Phi}_k$ and $\hat{\bm{\alpha}}_n$ yet there still exist a conceptual gap between the two concepts and the two approaches are often seen as orthogonal.

As we have seen, both settings lead to different beneficial aspects, compression, easy of projection and reconstruction or sparsity/clustering but more complex optimization problems, but, at a higher level, the signal processing community has always put a gap between those frameworks. As well put by Meyer in [CITE] one has to choose between encoding and representation.

We thus propose in this paper a novel approach allowing to have an over-complete dictionary yet given an input, a dynamic basis selection reduces it to an optimal complete dictionary without need for optimization in order to reconstruct. The selection is done without optimization in a forward manner leading to an efficient algorithm. This thus allows to inherit from all the properties induced by orthonormal basis while allowing adaptivity for better allowing to learn an over-complete basis with a nonlinear basis selection leading to an orthonormal basis when conditioned on the input. We also provide results from a low-dimensional manifold perspective and show that our approach perform nonparametric orbit estimation.
We validate on compression, dictionary learning tasks and clustering.

\section{Deep Residual Oja Network}
\subsection{Shallow Model}
Is it possible to learn a dictionary inheriting the benefits or complete and over-complete dictionaries ? We present one solution here. 
We first motivate the need for such a framework as well as present the general approach and notations. Throughout the next sections, the choice of the Oja name for the algorithm will become blatant for the reader.

Keeping the previously defined notation, we aim at learning an over-complete basis with the number of atoms defined by $FK$ with $F>1$ called the increase factor, note that $F=1$ lead to a complete basis, $K=D$ unless otherwise defined. By definition, the following projection-reconstruction scheme
\begin{equation}
    \hat{x}_n=\sum_k \frac{<x_,\hat{\phi}_k>}{||\hat{\phi}_k||^2}\hat{\phi}_k,
\end{equation}
can not reach an error $E_n<\epsilon$ and in fact $E_n$ increases with $F$. One way to resolve this issue comes from an optimal basis selection point of view leading to
\begin{equation}\label{loss1}
    E_n:=||x_n-\sum_k \delta_{n,k}\frac{<x_,\hat{\phi}_k>}{||\hat{\phi}_k||^2}\hat{\phi}_k|| < \epsilon,
\end{equation}
with $\delta{n,k}\in \{0,1\},\forall n,k$ representing a mask allowing to use a subset of the dictionary $\hat{\Phi}_{FK}$ we denote by $\rho_{\delta_{n,.}}[\hat{\Phi}_{FK}]$. 

\subsubsection{Error Bounds, Learning and Link with Oja Rule}
We first provide an error upper-bound for the proposed scheme $(S1)_{\kappa=1}$.
To simplify notations be also define by $\phi_\kappa(x_n)$ the atom defined by
\begin{equation}
\phi_\kappa(x_n)=\phi_{k'},k'=\argmax_k \frac{|<x_n,\phi_k>|^2}{||\phi_k||^2}.
\end{equation} 
\begin{theorem}
The error induced by $(S1)_{\kappa=1}$ is $||x_n||^2-\frac{|<x,\phi_\kappa(x_n)>|^2}{||\phi_\kappa(x_n)||^2}$ which is simply the reconstruction error from the best atom as since only one filter is used, nothing else is present.
\end{theorem}
\begin{proof}
By the incomplete basis theorem, there exists a basis s.t. it contains the atom $\phi_{\kappa}$, we denote by $\phi_k$ such a basis, with $k=1,...,D$, we thus have
\begin{align}
E_n=&|| x_n- \frac{<x,\phi_\kappa(x_n)>}{||\phi_\kappa(x_n)||^2}\phi_\kappa(x_n)||^2\nonumber\\
=&|| \sum_k\frac{<x_n,\phi_k>}{||\phi_k||^2}\phi_k- \frac{<x,\phi_\kappa(x_n)>}{||\phi_\kappa(x_n)||^2}\phi_\kappa(x_n)||^2 &&\text{ incomplete basis theorem} \nonumber \\
=&|| \sum_{k\not = \kappa}\frac{<x_n,\phi_k>}{||\phi_k||^2}\phi_k||^2\nonumber\\
=&\sum_{k\not = \kappa}\frac{|<x_n,\phi_k>|^2}{||\phi_k||^2}\nonumber\\
=&||x||^2-\frac{|<x_n,\phi_\kappa(x_n)>|^2}{||\phi_\kappa(x_n)||^2}&&\text{Parseval's Theorem}\nonumber\\
=&||x||^2\Big(1-\cos(\theta(x_n,\phi_\kappa(x_n))^2\Big)\label{en_eq}
\end{align}
And we have by definition $E_n\geq 0, E_n=0 \iff x_n=\phi_\kappa(x_n)$
\end{proof}
There is first a few comments on the loss and updates. First of all, the loss is closely related to a k-mean objective with cosine similarity and specifically spherical k-means which is the case where the centers and the data points are re-normalized to unit norm and has the objective to minimize
\begin{equation}
    \sum_n(1-cos(x_n,p_{c(n)})).
\end{equation}

\subsubsection{Learning and Oja Rule}
In order to learn the filter-bank $\Phi_K$, a common approach is to use an alternating scheme between finding the cluster belongings and optimizing the atoms w.r.t. this estimate. We first derive a gradient descent scheme to update the atoms and study some of its characteristics.

If we now denote by $n(k):=\{n:n=1,...,N|\kappa(x_n)=k\}$ be the collection of the sample index in cluster $k$, he resulting loss $E_{n(k)}:=\frac{\sum_{n \in n(k)}E_n}{Card(n(k))}$ we can now derive a gradient descent step as
\begin{equation}
    \phi_k^{(t+1)}(\lambda)=\phi_k^{(t)}-\lambda \frac{d E_{n(k)}}{d \phi_k},
\end{equation}
with 
\begin{align}
    \frac{d E_{n(k)}}{d \phi_k}&=\frac{1}{Card(n(k))}\sum_{n \in n(k)} \frac{2|<x_n,\phi_k>|}{||\phi_k||^2}\Big( \frac{|<x_n,\phi_k>|\phi_k}{||\phi_k||^2}-(-1)^{1_{<x_n,\phi_k> <0}}x_n \Big),\nonumber\\
    &=\frac{1}{Card(n(k))}\sum_{n \in n(k)} \frac{2<x_n,\phi_k>}{||\phi_k||^2}\Big( \frac{<x_n,\phi_k>\phi_k}{||\phi_k||^2}-x_n \Big).\label{oja_eq}
\end{align}
On the other hand, if one adopts an adaptive gradient step $\lambda$ per atom and point with one of the two strategies $\lambda_1,\lambda_2$ defined as
\begin{align}
    \lambda_1&=\frac{<x_n,\phi_k>}{2||x_n||^2}\\
    \lambda_2&=\frac{||\phi_k||^4}{2<x_n,\phi_k>^2}
\end{align}
then we have the 
\begin{align}
    \phi_k^{(t+1)}(\lambda_1)&=\phi_k^{(t)}-\frac{1}{\sum_n \cos(\theta(x_n,\phi_k))^2}\sum_{n \in n(k)}\cos(\theta(x_n,\phi_k))^2\Big( \frac{<x_n,\phi_k>\phi_k}{||\phi_k||^2}-x_n \Big),\label{eq_online1}\\
    \phi_k^{(t+1)}(\lambda_2)&=\frac{1}{Card(n(k))}\sum_{n \in n(k)}\frac{||\phi_k||^2}{<x_n,\phi_k>}x_n\label{eq_online2}
\end{align}
we thus end up with in the $\lambda_1$ case to a simple update rule depending on a weighted average of the points in the cluster based on their cosine similarity squared whereas for $\lambda_2$ we obtain a rule a la convex NMF with is a plain without update combination of the points available.

On the other hand, it is clear that minimizing $E_n$ from Eq. \ref{en_eq} is equivalent to maximizing $E^+_n=\frac{<x_n,\phi_\kappa(x_n)>}{||\phi_\kappa(x_n)||^2}$. As a result,  one can seize in Eq. \ref{oja_eq} the Oja rule as we can rewrite a GD update of $\phi_k$ as
\begin{align}
    \phi_k^{(t+1)}&=\phi^{(t)}_k+\gamma \frac{d E^+_n}{d \phi_k}(\phi^{(t)}_k)\\
    \phi_k^{(t+1)}&=\phi^{(t)}_k+\gamma \Big( x_n\frac{<x_n,\phi_k>}{||\phi_k||^2}-(\frac{<x_n,\phi_k>}{||\phi_k||^2})^2\phi_k \Big)\label{eq_online3}
\end{align}
known as the Oja rule.
SPEAK ABOUT OJA RULE. And in fact, the convergence of Oja rule toward the first eigenvector-eigenvalue is not surprising as $E^+_{n(k)}$ leads explicitly to 
\begin{equation}
    \phi_k=\argmax_{\phi}\frac{1}{Card(n(k))}\frac{\phi^TX(k)^TX(k)\phi}{\phi^T\phi},\label{eq_pca}
\end{equation}
which is known as the Rayleigh quotient and is a formulation of PCA leading to a one step global optimum being the greatest eigenvector-eigenvalue.
\begin{figure}[h]
\centering
\includegraphics[width=5in]{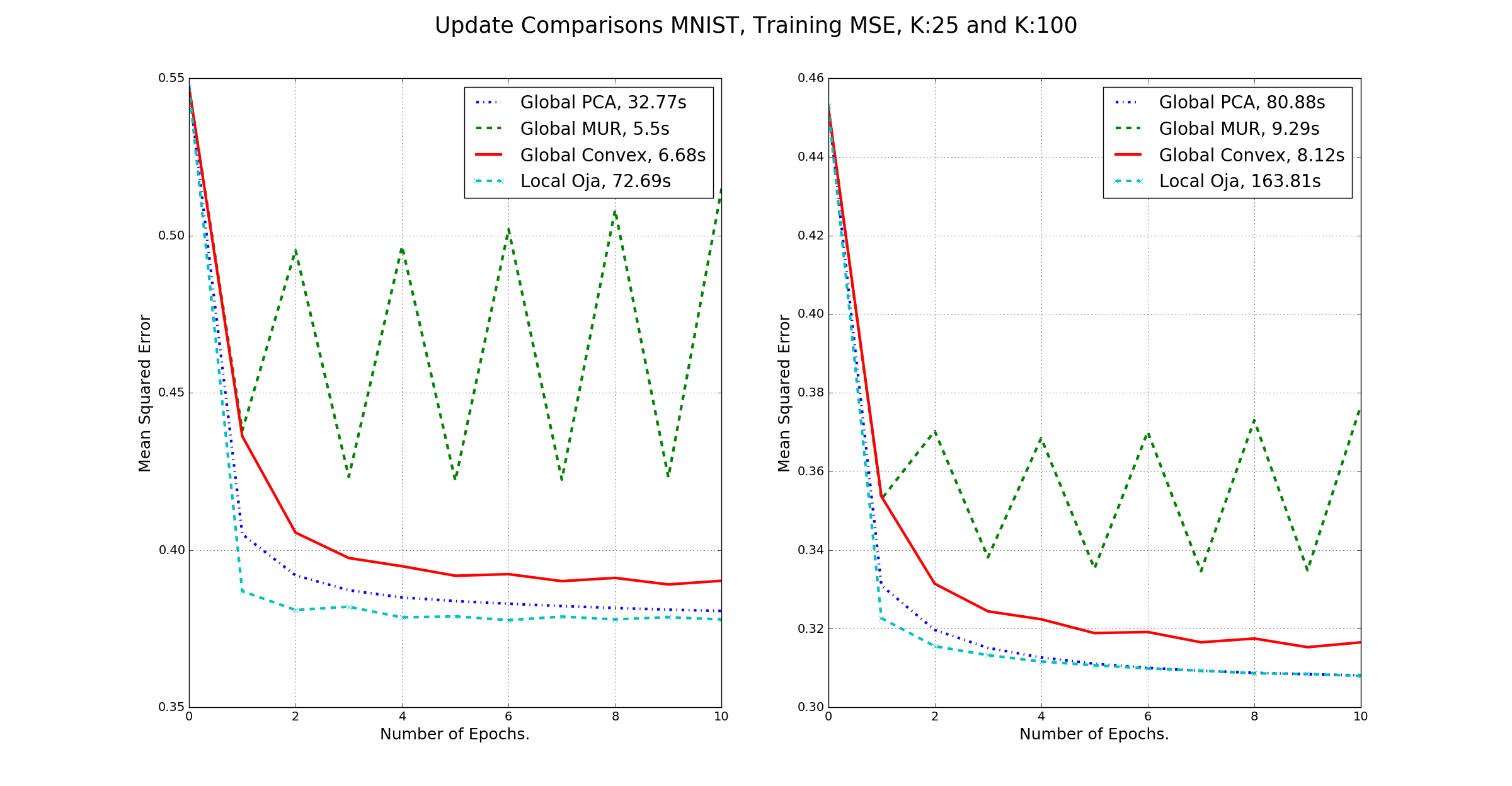}
\end{figure}

\begin{pseudocode}[doublebox]{Filter-Bank Learning strategy}{X,K}
\text{Initialize }\Phi_K\\
\WHILE \text{not converged} \DO
\BEGIN
\FOR k \GETS 1 \TO K \DO
\BEGIN
\text{Compute }n(k) \text{ with current }\Phi_K\\
\text{Update }\phi_k \text{with }n(k) \text{ and }X(k)\text{ according to Eq. \ref{eq_pca}}\\
\END\\
\END\\
\RETURN{\Phi_k}
\end{pseudocode}

\begin{pseudocode}[doublebox]{Online Filter-Bank Learning strategy}{X,K}
\text{Initialize }\Phi_K\\
\WHILE \text{not converged} \DO
\BEGIN
\FOR n \GETS 1 \TO N \DO
\BEGIN
\kappa = \argmax_k \frac{|<x_n,\phi_k>|^2}{||\phi_k||^2||x_n||^2}\\
\text{Update }\phi_\kappa \text{ according to Eq. \ref{eq_online1} or Eq.\ref{eq_online2} or Eq.\ref{eq_online3}}\\
\END\\
\END\\
\RETURN{\Phi_k}
\end{pseudocode}

\begin{theorem}
If the distribution of the $x_n$ in the space is uniformly distributed, the optimal over-complete basis for $(S1)_{kappa=1}$ is thus the one corresponding of a quantization of the sphere, it is unique up to a change of sign and global rotations (same applied to each atom). For the $2$-dimensional case it is easy to see that the maximum error for any given point $x_n$ is exactly upper-bounded by $||x||^2\Big(1-\cos(\frac{\pi}{2FK})^2\Big)$ if $FK$ growths exponentially .(?? CHECK POWER OF HIGH DIMENION COSINE decompose as union of 2D spaces)
\end{theorem}
\begin{proof}
For 2D we now the upper bound is $||x||^2\Big(1-\cos(\frac{\pi}{2FK})^2\Big)$ with $FK$ atoms, we thus rewrite 
\begin{align*}
||x-\hat{x}||^2=\sum_{d=1}^{D/2}||x_d-\hat{x}_d||^2
\end{align*}
and see that in order to have the upper bound for each subspace we need the cartesian product of all the subspace basis $\Phi_{FK}$ leading to $FK$ atoms. Thus one need to grow exponentially the number of atom w.r.t the dimension to have a loss increasing linearly.
\end{proof}
\begin{figure}[h]
\centering
\includegraphics[width=5in]{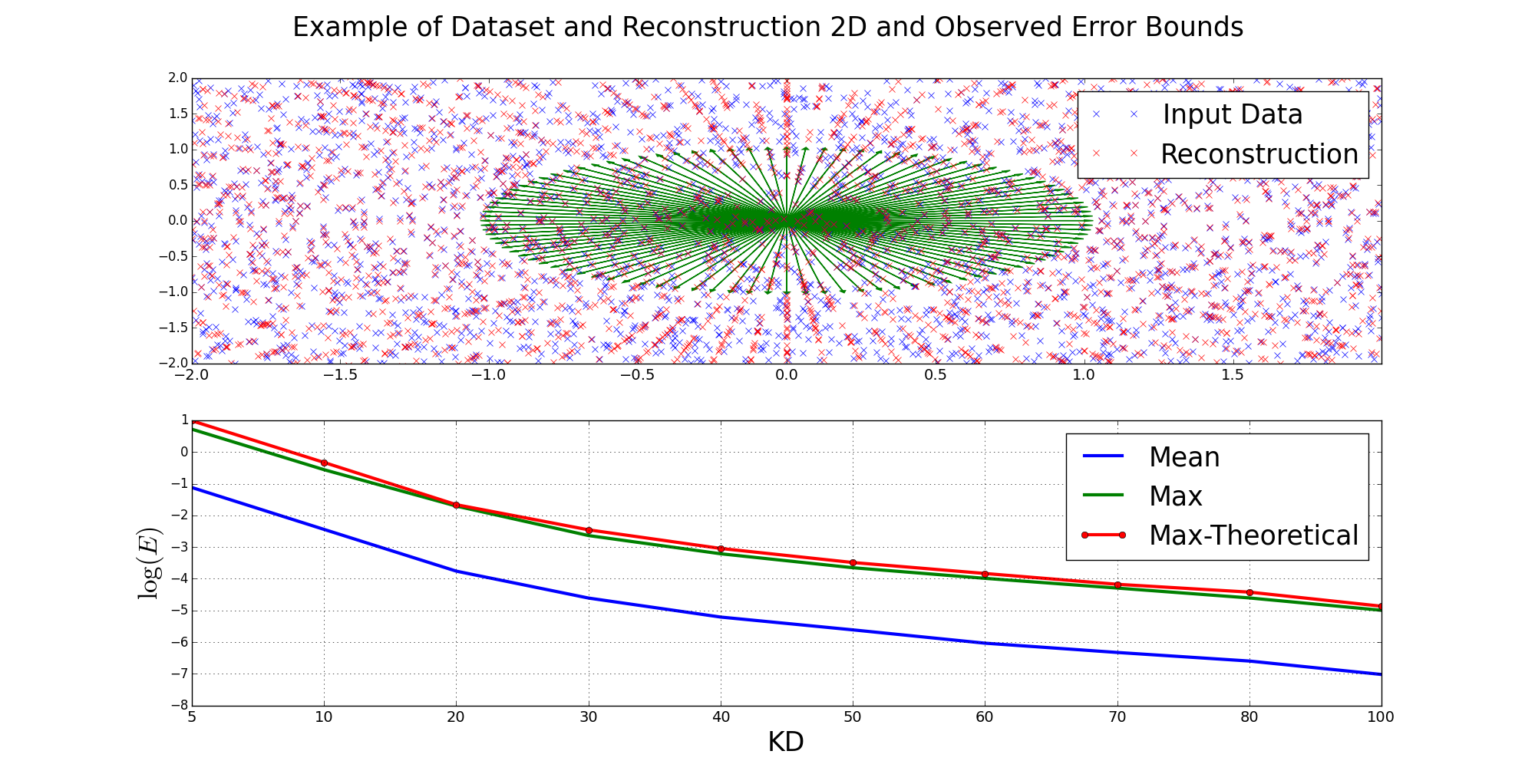}
\end{figure}

However this pessimistic upper-bound assumes the worst possible scenario: uniform distribution of the data point in the space $\mathbb{R}^D$ which in general is false. In fact, many dataset have inherent structures and at least lye only in a small subset sometime regular of $\mathbb{R}^D$.
In general, data might be living in unions of subsets and thus providing a general strategy or optimal basis is a priori more complex thus pushing the need to learn the atoms as it is done in general for k-mean applications.
\begin{theorem}\label{th5}
A sufficient condition for $(S1)_{kappa=1}$ to be optimal is that the data are already clustered along $FD$ lines, or orbits ??? :)
\end{theorem}
We now present one way to tackle the curse of dimensionality in the next section

\subsection{Multiple Atoms}
\begin{equation}
    E_n=|| x_n- \sum_{k=1}^K\frac{<x,\phi^k_\kappa(x_n)>}{||\phi^k_\kappa(x_n)||^2}\phi^k_\kappa(x_n)||^2
\end{equation}
For learning atom after atom a la coordinate ascent we have that 
\begin{align*}
    \hat{\phi}^{k'}_j=& \argmin_{\phi^{k'}_j}\sum_n E_n\\
    =&\argmin_{\phi^{k'}_j}\sum_{n \in n(k,j)} || \Big(x_n- \sum_{k=1,\not = k'}^K\frac{<x,\phi^k_\kappa(x_n)>}{||\phi^k_\kappa(x_n)||^2}\Big)-\frac{<x,\phi^{k'}_j>}{||\phi^{k'}_j||^2}\phi^{k'}_\kappa(x_n)||^2
\end{align*}
as we showed in the last section we end up with the same update rule but with the input being substracted by the other used atoms. Thus we still perform PCA but on the input minus the other atoms. Note that this ensures orthogonality between the atoms.

\subsection{From Shallow to Deep Residual for better Generalization Error Bounds and Combinatorially Large Dictionaries}
\begin{figure}[t]
\centering
\includegraphics[width=5in]{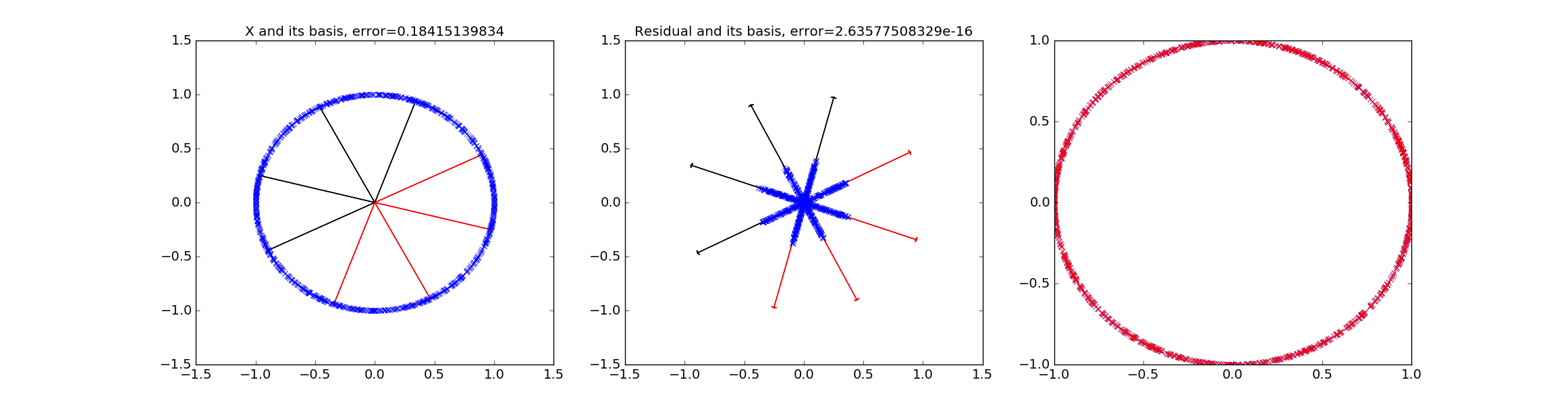}
\end{figure}

We now consider the analysis of the generalization performance on out of bag observations as well as the problem of having really large dataset.

\begin{theorem}
If we suppose an finite training set and an Oja Network with sufficient filters in order to reach a training error of $0$ then we know that the generalization error is directly lower-bounded by how close the testing and training examples are. In fact
\begin{equation}
    E_{new}\propto \cos(\theta(x_\kappa,x_{new})),\kappa=\argmin_n \cos(\theta(x_n,x_{new})).
\end{equation}
\end{theorem}

The proof is straightforward as we know the network is able to perfectly reconstruct the training set and only the training set.
As a result, if the training set is well and uniformly sampled among the space of possible observations, a shallow Oja Network can be considered as optimal also for the testing set.
However, and especially for computer vision task, it is well known that every observation is very far form each other in term of distance when dealt with in the pixel domain, also, requiring a proper sampling of the space of images is clearly outrageous.
We thus now present the result motivating deep architectures in general including the Oja Network.

\begin{align}
R^{(l)}_n=&R^{(l-1)}_n-\frac{<R^{(l-1)}_n,\phi_\kappa^{(l)}>}{||\phi_\kappa^{(l)}||^2}\phi_\kappa^{(l)}, \kappa = \argmax_k \frac{|<R^{(l-1)}_n,\phi_k^{(l)}>|}{||R^{(l-1)}_n||^2||\phi_k^{(l)}||^2} \nonumber \\
R^{(0)}_n=&x_n
\end{align}

as a result as soon as the input and the template are not orthogonal there is convergence. 
\begin{theorem}
Since we have by definition that the selected atom is the one with smaller angle, if it is $0$ it means that the input $R^{(l-1)}$ is orthogonal to all the learned dictionary $\Phi^{(l)}$
\begin{equation}
\cos \Big(\theta (R^{(l-1)}_n,\phi^{(l)}_\kappa) \Big)^2=0 \iff R^{(l-1)} indep \phi^{(l)}_k \forall k,
\end{equation}
and thus they live in two orthogonal spaces.
TO PROVE : ALL THE NEXT ONES ARE ALSO 0 
\end{theorem}

\begin{theorem}
The residual decreases exponentially w.r.t. the depth of the model.
\end{theorem}
\begin{proof}
\begin{align}
||R^{(l)}_n||^2=&||R^{(l-1)}_n-\frac{<R^{(l-1)}_n,\phi_\kappa^{(l)}>}{||\phi_\kappa^{(l)}||^2}\phi_\kappa^{(l)}||^2 \nonumber \\
=&||R^{(l-1)}_n||^2-\frac{<R^{(l-1)}_n,\phi_\kappa^{(l)}>^2}{||\phi_\kappa^{(l)}||^2}\nonumber \\
=&||R^{(l-1)}_n||^2\Big(1-\cos \Big(\theta(R^{(l-1)}_n,\phi_\kappa^{(l)}) \Big)^2\Big)\\
=&||x_n||^2\prod_{l=1}^l\Big(1-\cos \Big(\theta(R^{(l-1)}_n,\phi_\kappa^{(l)}) \Big)^2 \Big)
\end{align}
\end{proof}

The final template can be flattened via
\begin{align}
T_n =& \sum_l \frac{<R^{(l-1)}_n,\phi_\kappa^{(l)}>}{||\phi_\kappa^{(l)}||^2}\phi_\kappa^{(l)}\\
=&\sum_l P^{(l)}_n
\end{align}

\subsubsection{Learning}
Computing the gradient finds a great recursion formula we define as follows:
\begin{align}
    A_{i,j}=\left\{ \begin{matrix}
         0 \iff j<i\\
         \frac{<R^{(i-1)},\phi^{(i)}_\kappa>I_d+R^{(i-1)}\phi^{(i)}_\kappa}{||\phi^{(i)}_\kappa||^2}+\frac{2<R^{(i-1)},\phi^{(i)}_\kappa>\phi^{(i)}_\kappa\phi^{(i)T}_\kappa}{||\phi^{(i)}_\kappa||^4} \iff i=j\\
         A_{i,j-1}-\frac{\phi^{(j)}_\kappa\phi^{(j)T}_\kappa A_{i,j-1}}{||\phi^{(j)}_\kappa||^2}\iff j>i
    \end{matrix} \right.
\end{align}
thus $A_{i,j} \in \mathbb{R}^{D \times D}$ then we have
\begin{align}
    \mathcal{L}_n=&||R^{(L)}_n||,\\
    \frac{\d \mathcal{L}^2}{\d \phi^{(l)}_\kappa}=&2R^{(L)}_n\frac{\d R^{(L)}_n}{\d \phi^{(l)}_\kappa}\\
    \end{align}
However as we will see below we have a nice recursive definition to compute all those derivatives, in fact
\begin{equation}\text{Init. }
\begin{cases}
    \frac{\d P^{(l)}_n}{\d \phi^{(l)}_\kappa}&=\frac{<R^{(l-1)}_n,\phi^{(l)}_\kappa>I_d+R^{(l-1)_n}\phi^{(l)}_\kappa}{||\phi^{(l)}_\kappa||^2}+\frac{2<R^{(i-1)}_n,\phi^{(l)}_\kappa>\phi^{(l)}_\kappa\phi^{(l)T}_\kappa}{||\phi^{(l)}_\kappa||^4},\\
    \frac{\d R^{(l)}_n}{\d \phi^{(l)}_\kappa}&=-\frac{\d P^{(l)}_n}{\d \phi^{(l)}_\kappa}
\end{cases}
\end{equation}

\begin{equation}\text{Recursion }
\begin{cases}
    \frac{\d P^{(l+1)}_n}{\d \phi^{(l)}_\kappa}&=\frac{\phi^{(l+1)}_\kappa\phi^{(l+1)^T}_\kappa}{||\phi^{(l+1)}_\kappa||^2}\frac{\d R^{(l)}_n}{\d \phi_\kappa^{(l)}},\\
    \frac{\d R^{(l+1)}_n}{\d \phi^{(l)}_\kappa}&=\frac{\d R^{(l)}_n}{\d \phi^{(l)}_\kappa}-\frac{\d P^{(l+1)}_n}{\d \phi^{(l)}_\kappa}
\end{cases}
\end{equation}

\begin{pseudocode}[doublebox]{Residual Oja Network}{X,K}
R_n \GETS X_n, \forall n \\
\FOR l \GETS 1 \TO L \DO
\BEGIN
\text{Initialize }\Phi^{(l)}_K \text{ from }R\\
\WHILE \text{not converged} \DO
\BEGIN
\FOR k \GETS 1 \TO K \DO
\BEGIN
\text{Compute }n(k) \text{ with current }\Phi^{(l)}_K\\
\text{Update }\phi^{(l)}_k \text{with }n(k) \text{ and }R(k)\text{ according to Eq. \ref{eq_pca}}\\
\END\\
\END\\
R_n = (R_n-\frac{<R_n,\phi^{(l)}_\kappa>}{||\phi^{(l)}_\kappa ||^2}\phi^{(l)}_\kappa)
\END\\
\RETURN{\Phi^{(l)}_k, \forall l}
\end{pseudocode}

\begin{pseudocode}[doublebox]{Online Residual Oja Network}{X,K}
R_n \GETS X_n, \forall n \\
\FOR l \GETS 1 \TO L \DO
\BEGIN
\text{Initialize }\Phi^{(l)}_K \text{ from }R\\
\WHILE \text{not converged} \DO
\BEGIN
\FOR n \GETS 1 \TO N \DO
\BEGIN
\kappa = \argmax_k \frac{|<R_n,\phi^{(l)}_k>|^2}{||\phi^{(l)}_k||^2||R_n||^2}\\
\text{Update }\phi^{(l)}_\kappa \text{ according to Eq. \ref{eq_online1} or Eq.\ref{eq_online2} or Eq.\ref{eq_online3}}\\
\END\\
\END\\
R_n = (R_n-\frac{<R_n,\phi^{(l)}_\kappa>}{||\phi^{(l)}_\kappa ||^2}\phi^{(l)}_\kappa)
\END\\
\RETURN{\Phi^{(l)}_k, \forall l}
\end{pseudocode}

\begin{theorem}
With a Deep (Oja) Network, the previously presented lower-bound of the generalization error becomes an upper-bound.
\end{theorem}
In addition of guaranteeing better generalization errors through depth, we also benefit from another gain. The depth as we will see allows for an exponential amount of possible templates to be constructed perfectly with only a linear increase in the number of learned parameters.

\begin{lstlisting}[caption=Input to Mask]
####################
#  INPUT: X(N,channels,Ix,Jx),w(n_filters,channels,Iw,Jw)
####################
k    = T.nnet.conv.conv2d(x,w,stride=stride,
        border_mode='valid',flip_filters=False,input_shape=(N,channels,Ix,Jx),
        filters_shape=(n_filters,channels,Iw,Jw))#(N,n_filters,(Ix-Iw)/stride+1,(Jx-Jw)/stride+1)
output = ((k>0)*2-1)*T.signal.pool.max_pool_2d_same_size(
		theano.tensor.abs_(k).dimshuffle([0,2,3,1]),
		(1,n_filters)).dimshuffle([0,3,1,2])#(N,n_filters,(Ix-Iw)/stride+1,(Jx-Jw)/stride+1)
mask = T.switch(T.eq(output,0),0,1)#(N,n_filters,(Ix-Iw)/stride+1,(Jx-Jw)/stride+1)
\end{lstlisting}

\begin{lstlisting}[caption=Reconstruction]
####################
# INPUTS: Z(N,n_filters,Iz,Jz),w(n_filters,channels,Iw,Jw),stride
####################
dilated_output = T.set_subtensor(T.zeros((N,n_filters,(Iz-1)*stride),(Iz-1)*stride),
		dtype='float32')[:,:,::stride,::stride],Z)#(N,n_filters,Ix-Iw+1,Jx-Jw+1)
rec = T.nnet.conv.conv2d(dilated_Z,w.dimshuffle([1,0,2,3]),stride=1,
        border_mode='full',flip_filters=False)#(N,channels,Ix,Jx)
\end{lstlisting}

\begin{lstlisting}[caption=Mask to Grad]
###################
#  INPUT : rec(N,C,Ix,Jx),mask(N,n_filters,Iz,Jz),Iw,Jw
###################
d_W,outputs=theano.scan(fn=lambda acc,i,X,mask:
        acc+conv2d(rec[i].dimshuffle([0,'x',1,2]),mask[i].dimshuffle([0,'x',1,2]),
        input_shape=(C,1,Ix,Jx),
        filter_shape=(n_filters,1,Iz,Jz)).dimshuffle([1,0,2,3]),
        sequences=[theano.tensor.arange(N,dtype='int32')],
        non_sequences=[rec,mask],outputs_info = T.zeros((n_filters,C,Iw,Jw),dtype='float32'))
d_W = d_W[-1]
\caption{algo}
\end{lstlisting}

\begin{figure}[h]
\centering
\includegraphics[width=5in]{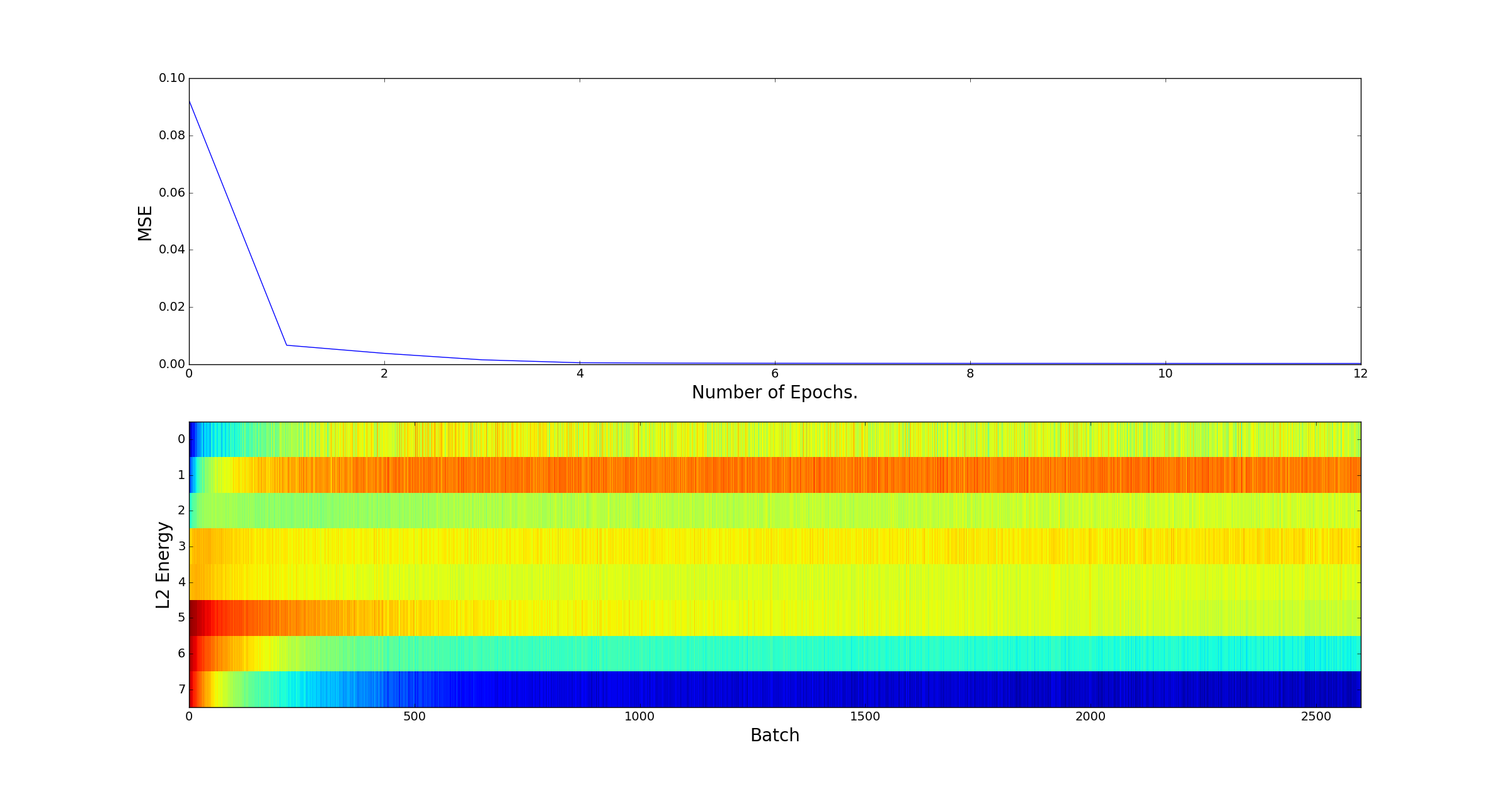}
\caption{Top : evolution of the reconstruciton error w.r.t. epochs. Bottom: evolution of the energy captured per level during training. At first stage the last levels capture everything since random initialization makes the global filters almost orthogonal to images, during training global filters learn to capture the low frequencies. Since it is known that natural images have a $1/f$ decay of energy over frequencies $f$ we can see that the final energy repartition is indeed bigger for low-frequency/global filters and go down for smaller filters.}
\end{figure}
\begin{figure}[h]
\centering
\includegraphics[width=5in]{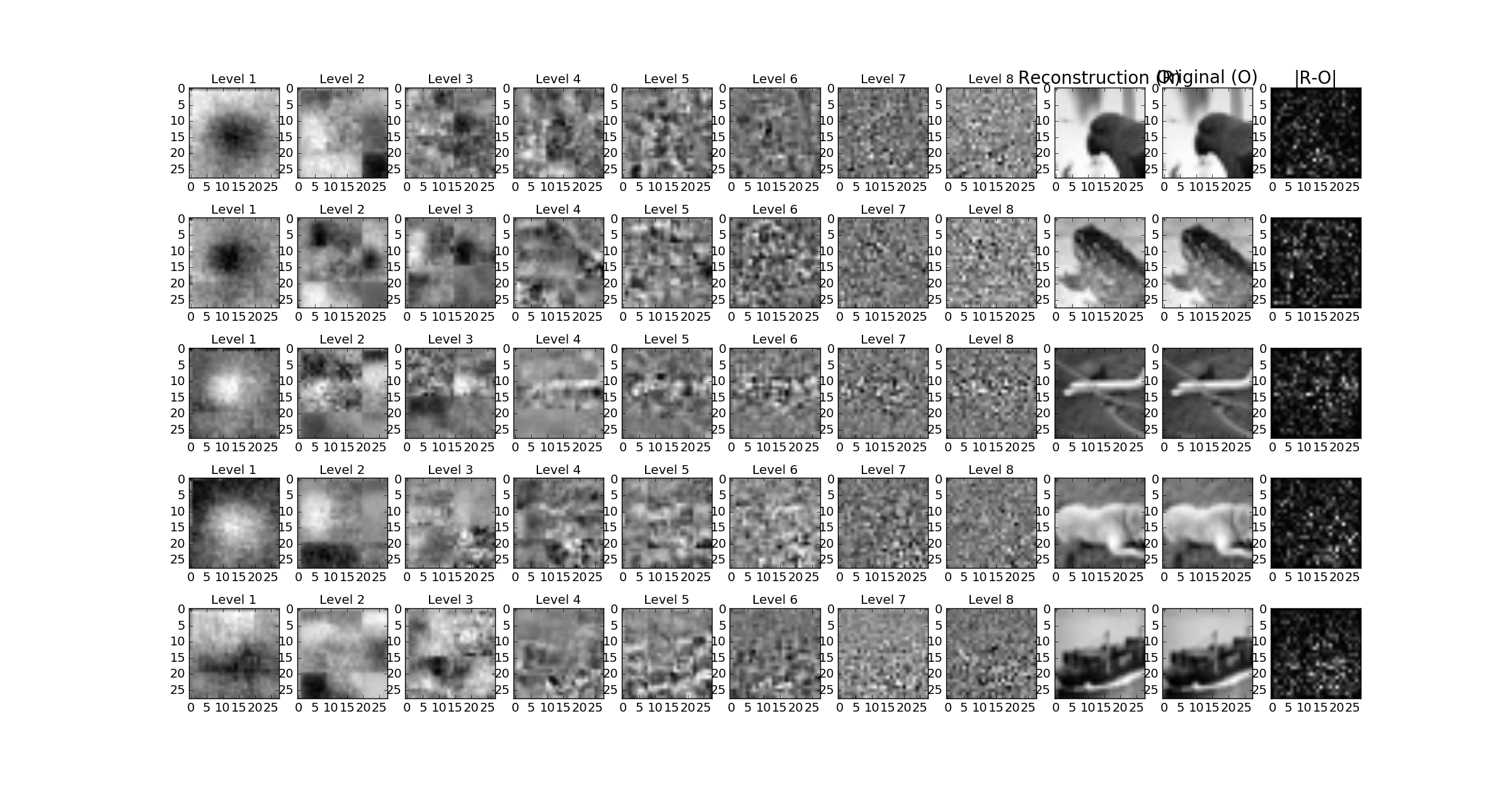}
\caption{Example of decomposition and reconstruction of some CIFAR10 images. From right to left is the final residual (reconstruction minus original), the original image, the reconstructed images and then all the decompositions, on the left is the global/large one. Summing elementwise columns $1$ to $8$ leads to column $9$ the reconstrued input.}
\end{figure}

\subsection{Previous Work}
The proposedm ethod can be seen as bilinear sparse coding with one-hot latent vector $y$ \cite{grimes2005bilinear} for the case of only one filter used. There is also a direct link with the probabilistic version of this work, namely mixture of PPCA \cite{tipping1999probabilistic,tipping1999mixtures}, as here we are in a ''hard clustering'' case similarly to k-means versus GMM. 
By the selection of the best matching atom, we find some links with matching pursuit \cite{tropp2007signal,pati1993orthogonal} and also  locality sensitive hashing \cite{indyk1998approximate,johnson1984extensions} especially in the  cosine similarity distance.

This problem can also be seen from a best basis selection point of view coupled with dictionary learning. 
Popular examples with varying degrees of computational overhead include convex relaxations such
as $L1$-norm minimization \cite{beck2009fast,candes2006robust,tibshirani1996regression},  greedy approaches like orthogonal matching pursuit (OMP)
\cite{pati1993orthogonal,tropp2004greed}, and many flavors of iterative hard-thresholding (IHT) \cite{blumensath2009iterative,blumensath2010normalized}
Variants of these algorithms find practical relevance in numerous disparate domains, including feature selection \cite{cotter2002sparse,figueiredo2002adaptive}, outlier removal \cite{candes2005decoding,ikehata2012robust}, compressive sensing \cite{baraniuk2007compressive}, and source localization \cite{baillet2001electromagnetic,model2006signal}

\section{Conclusion}
We presented a hierarchical version of the deterministic mixture of PCA and presented results on CIFAR10 images. We also provided algorithms allowing GPU computation for large scale dataset and speed. The main novelty comes from the deterministic formulate of the probabilistic mixture of PCA which allows easier use as it is known in general that MPPCA is unstable for large scale problems. From this we derived its hierarchical residual version which inherits many benefits and allow for exponentially good reconstruction w.r.t. the depth. We also believe that this residual approach allowing to learn orthogonal spaces will lead to interesting dictionary learning mixing for example residual networks with this approach.
\bibliography{ref}
\bibliographystyle{plain}

\end{document}